\documentclass[journal]{IEEEtran}
\ifCLASSINFOpdf
\else
\fi
\usepackage{amssymb}
\usepackage{slashbox}
\usepackage{mathrsfs}
\usepackage{mathrsfs}
\usepackage{mathrsfs}
\usepackage{amsfonts}
\usepackage{amsfonts}
\usepackage{amsfonts}
\usepackage{mathrsfs}
\usepackage{amsmath}
\usepackage{tabularx}
\usepackage{amsthm}

\usepackage{graphicx}
\usepackage{multirow}
\usepackage{caption2}
\usepackage{float}
\usepackage{booktabs}
\usepackage{algorithm}
\usepackage{algorithmic}
\usepackage{subfigure}
\usepackage{color}

\newtheorem{theorem}{Theorem}[section]
\renewcommand{\theequation}{\thesection.\arabic{equation}}
\numberwithin{equation}{section}

\newtheorem{lemma}[theorem]{Lemma}

\begin{document}
%
\title{Greedy Criterion  in Orthogonal Greedy Learning }
%
%
%

\author{{Lin Xu,~Shaobo Lin,~Jinshan Zeng,~Xia Liu
    and~Zongben Xu}
\thanks{L. Xu, X. Liu and Z. B. Xu are with the Institute for Information and System
    Sciences, Xi'an Jiaotong
    University, Xi'an 710049, China.}
\thanks{S. B. Lin is with the College of Mathematics and Information Science, Wenzhou
    University, Wenzhou 325035, China.}
\thanks{J. S. Zeng is with the College of Computer Information Engineering, Jiangxi Normal University, Nanchang, Jiangxi 330022, China.}
\thanks{}}

%
%

\markboth{Journal of \LaTeX\ Class Files,~Vol.~14, No.~8, August~2015}%
{Shell \MakeLowercase{\textit{et al.}}: Bare Demo of IEEEtran.cls for IEEE Journals}
%



\maketitle

\begin{abstract}
Orthogonal greedy learning (OGL) is a stepwise learning scheme that starts with selecting a new atom from a specified dictionary via the  steepest gradient descent (SGD) and  then builds the estimator through orthogonal projection. In this paper, we  find that SGD is not the unique greedy criterion and  introduce a new greedy criterion, called ``$\delta$-greedy threshold'' for learning. Based on the new greedy criterion, we derive an adaptive termination  rule for OGL. Our theoretical study shows that the new
learning scheme can achieve the existing (almost) optimal learning rate of OGL. Plenty of numerical experiments are provided to support that the new scheme can achieve almost optimal generalization performance, while requiring less computation than OGL.
\end{abstract}


\begin{IEEEkeywords}
Supervised learning,  greedy algorithms, orthogonal greedy learning, greedy criterion, generalization capability.
\end{IEEEkeywords}


%
\IEEEpeerreviewmaketitle

\section{Introduction}
%
%
%
%
\IEEEPARstart{S}{upervised} learning  focuses on synthesizing a
function to approximate an underlying relationship between  inputs
and  outputs based on finitely many input-output samples.
Commonly, a system tackling supervised learning problems is  called as a
learning system. A standard learning system usually comprises a
hypothesis space, an optimization strategy and a learning algorithm.
The hypothesis space is a family of parameterized
functions  providing a candidate set of estimators,
the optimization strategy formulates an optimization problem to define the estimator based on  samples,
and the learning algorithm is an inference procedure that numerically solves the optimization problem.

Dictionary learning is a special
learning system,  whose
hypothesis spaces are  linear combinations  of atoms in some given dictionaries.
Here, the dictionary denotes
a family of base learners \cite{Temlaykov2008}. 
For such type hypothesis spaces, many regularization schemes such as the bridge
estimator \cite{Armagan2009}, ridge estimator \cite{Golub1979} and
Lasso estimator \cite{Tibshirani1995} are common used optimization strategies.
When the scale of dictionary is moderate (i.e., about hundreds of atoms), these optimization strategies can be effectively realized by various
learning algorithms such as the regularized least squares algorithms
\cite{Wu2006}, iterative thresholding algorithms
\cite{Daubechies2004} and iterative re-weighted algorithms
\cite{Daubechies2010}. However,  when presented large input 
dictionary, a large portion of the aforementioned learning
algorithms are time-consuming and even worse, they may cause the
sluggishness of the corresponding learning systems.

Greedy learning  or,  more specifically,  learning by greedy type
algorithms, provides a possible way to circumvent the drawbacks of
regularization methods \cite{Barron2008}.  Greedy algorithms
are stepwise inference processes that start from a null model and
solve heuristically   the problem heuristically of making the locally
optimal choice at each step with the hope of finding a global
optimum. Within moderate number of iterations, greedy
algorithms possess charming computational advantage compared with
the regularization schemes \cite{Temlaykov2008}. This property
triggers avid research activities of greedy algorithms in
signal processing \cite{Dai2009,Kunis2008,Tropp2004}, inverse
problems \cite{Donoho2012,Tropp2010}, sparse approximation
\cite{Donoho2007,Temlaykov2011} and  machine learning
\cite{Barron2008,Chen2013a,Lin2013a}.

\subsection{Motivations of greedy criteria}

Orthogonal greedy learning (OGL) is a special greedy learning
strategy. It selects a new atom based on SGD in each iteration 
and then constructs an estimator through 
orthogonal projecting to subspaces spanned by the selected atoms.  
It is well known that SGD needs to traverse the whole dictionary, which leads
to an insufferable computational burden when the scale of dictionary
is large. Moreover, as the samples are  noised, the generalization
capability  of OGL is  sensitive to the number of iterations. In
other words, due to the SGD criterion, a slight turbulence of the
number of atoms may lead to a great change of the generalization
performance.

To overcome the above problems of OGL, a natural idea is to
re-regulate the criterion to choose a new atom by taking the
``greedy criterion'' issue into account. The  Fig.
\ref{IdeaGC}  is an intuitive description to quantify the greedy
criterion,  where   $r_k$ represents the residual at the  $k$-th
iteration,  $g$ is an arbitrary atom from the dictionary and  $\theta$
is the included angle between $r_k$ and $g$. In Fig. \ref{IdeaGC} (a),
both  $r_k$ and  $g$  are normalized to the unit ball due to the greedy
criterion focusing on the orientation rather than magnitude. The
cosine of the angle $\theta$ (cosine similarity) is used to quantify
the greedy criterion.
As shown in Fig. \ref{IdeaGC} (b),
the atom  $g_k$  possessing the smallest   $\theta$   is regarded to be
the greediest one at each iteration in OGL.

\begin{figure}[htb]
    \centering
    \includegraphics[height=3.5cm,width=7cm]{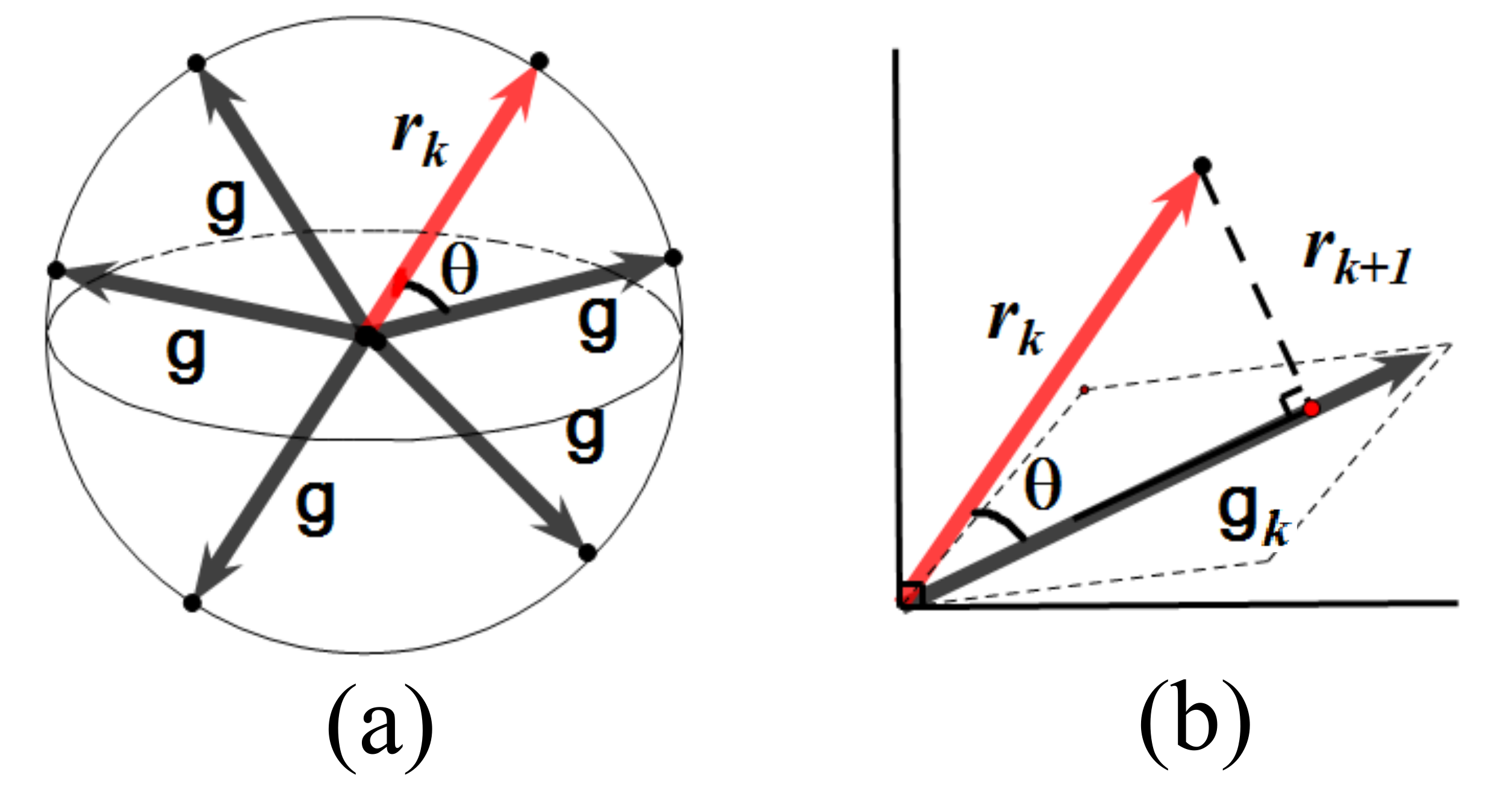}
    \caption{An intuitive description of the greedy criterion. (a) Normalize the current residual  $r_k$ and atoms $g$ to the unit ball. (b) The atom  $g_k$  possessing the smallest   $\theta$   is regarded to be
    the greediest one at each iteration.} \label{IdeaGC}
\end{figure}

Since the  greedy criterion can be quantified by the cosine
similarity, a preferable way to circumvent the aforementioned
problems of OGL is to weaken the level of greed by   thresholding
the cosine similarity. In particular, other than  traversing the
dictionary, we can select the first atom satisfying the thresholding
condition. Such a method essentially reduces the computation cost of
OGL and makes the learning process more stable.

\subsection{Our contributions}

Different from  other three issues, the ``greedy criterion'' issue,
to the best of our knowledge, has not been noted for the learning
purpose. The aim of the present paper is to reveal the importance
and necessity of studying the ``greedy criterion'' issue in OGL. The
main contributions can be summarized as follows.

$\bullet$ We argue that SGD is not the unique criterion for
OGL. There are many other greedy criteria in greedy learning, which
possess  similar learning performance  as SGD. 

$\bullet$ We use a new greedy criterion called the
``$\delta$-greedy threshold'' to quantify the level of greed in OGL.
Although  a similar criterion has already been used in greedy
approximation \cite{Temlaykov2008a},
the novelty of translating it
into greedy learning is that using this criterion can significantly
accelerate the learning process.
We can also prove that, if the number
of iteration is appropriately specified, then OGL with the ``$\delta$-greedy threshold''
can reach the existing
(almost) optimal learning rate of OGL \cite{Barron2008}.

$\bullet$ Based on the ``$\delta$-greedy threshold'' criterion, we
propose  an adaptive terminate
rule  for OGL and then provide a
complete learning system called  $\delta$-thresholding orthogonal
greedy learning ($\delta$-TOGL). Different from  classical
termination  rules that devote to searching the appropriate number of
iterations based on the bias-variance balance principle
\cite{Barron2008,Xu2014}, our study implies that the balance can
also be attained through setting a suitable greedy
threshold. This phenomenon reveals the essential importance of the
``greedy criterion'' issue. We also present the theoretical
justification of $\delta$-TOGL.

$\bullet$ We carefully analyze the generalization performance and
computation  cost of $\delta$-TOGL, compared with other  popular
learning strategies such as the   pure greedy learning (PGL)
\cite{Barron2008, Temlaykov2008}, OGL, regularized least squares
(RLS) \cite{Hoerl1970} and fast iterative shrinkage-thresholding
algorithm (FISTA) \cite{Beck2009} through plenty of numerical
studies. The main advantage of $\delta$-TOGL is that it
can reduce the computational cost without sacrificing the
generalization capability. In many applications,  it can learn
hundreds of times faster than conventional methods.

\subsection{Organization}
The rest of the paper is organized as follows. In
Section 2,
we present a brief introduction of statistical learning theory and
greedy learning. In Section 3, we introduce the ``$\delta$-greedy
threshold'' criterion in OGL and provide its feasibility
justification. In Section 4, based on the ``$\delta$-greedy
threshold'' criterion, we propose an adaptive termination  rule and the
corresponding $\delta$-TOGL system. The theoretical feasibility of
the $\delta$-TOGL system is also given in this section. In Section
5, we present numerical simulation experiments to verify our
arguments.  In Section 6, $\delta$-TOGL is tested with real-world data. In Section 7, we provide the detailed proofs of the main
results. Finally, the conclusion is drawn in the last section.

\section{Preliminaries}
In this section, we present some preliminaries  to serve as the basis for the following sections.

\subsection{Statistical learning theory}
Suppose that the samples $\mathbf{z}=(x_{i},y_{i})_{i=1}^{m}$ are
drawn independently and identically from $Z:=X\times Y$ according to
an unknown probability distribution $\rho $ which admits the
decomposition
\begin{equation}
\rho (x,y)=\rho _{X}(x)\rho (y|x).
\end{equation}

 Let $f:X\rightarrow Y$ be an approximation of the underlying relation between the input and output spaces.
A commonly used measurement of the quality of $f$
is the generalization error, defined by
\begin{equation}
\mathcal{E}(f):=\int_{Z}(f(x)-y)^{2}d\rho ,
\end{equation}
which is minimized by the regression function \cite{Cucker2001}
\begin{equation}
f_{\rho }(x):=\int_{Y}yd\rho (y|x).
\end{equation}
The goal of learning is to find a best approximation of the
regression function $f_{\rho }$.

Let $L_{\rho _{_{X}}}^{2}$ be the Hilbert space of $\rho _{X}$
square integrable functions on $X$, with norm $\Vert \cdot \Vert
_{\rho }.$
It is known that, for every $f\in L_{\rho _{X}}^{2}$,
it holds that
\begin{equation}
\mathcal{E}(f)-\mathcal{E}(f_{\rho })=\Vert f-f_{\rho }\Vert _{\rho
}^{2}.
\end{equation}
%
%

Without loss of generality, we  assume $y \in [-M,M]$ almost surely. Thus,
it is reasonable to truncate the estimator to $[-M,M]$. That is, if
we define
\begin{equation}
\pi_Mu:=\left\{\begin{array}{l l}
u,       & \mbox{if}\ |u|\leq M  \\
M\text{sign}(u),& \mbox{otherwise}
\end{array}
\right.
\end{equation}
as the truncation operator,  where $\text{sign}(u)$ represents the sign function of $u$, then
\begin{equation}
\|\pi_Mf_{\bf z}-f_\rho\|^2_\rho\leq \|f_{\bf z}-f_\rho\|^2_\rho.
\end{equation}

\subsection{Greedy learning}

Four most important elements  of greedy learning are  
\textit{dictionary selection}, \textit{greedy criterion}, \textit{iterative strategy} and \textit{termination  rule}. This is essentially different from  greedy approximation which   focuses only on  \textit{dictionary selection} and
\textit{iterative format} issues \cite{Temlaykov2008}. Greedy learning concerns not only the approximation capability, but also the cost, such as the model complexity,  which should pay to achieve a specified approximation accuracy. In a nutshell, greedy learning can be regarded as a four-issue learning scheme.

$\bullet$  \textit{Dictionary selection} : this issue devotes to selecting a suitable dictionary for a given learning task. As a classical topic of greedy approximation, there are a great deal of dictionaries available to greedy learning. Typical examples include the greedy basis \cite{Temlaykov2008}, quasi-greedy basis
\cite{Temlaykov2003}, redundant dictionary \cite{Devore1996},
orthogonal basis \cite{Temlyakov1998}, kernel-based sample dependent dictionary \cite{Chen2013,Lin2013a} and tree  \cite{Friedman2001}.

$\bullet$ \textit{Greedy criterion} : this issue regulates the
criterion to choose a new atom   from the dictionary in each greedy
step. Besides the widely used steepest gradient descent (SGD) method
\cite{Devore1996}, there are also many  methods such as the weak
greedy \cite{Temlaykov2000}, thresholding greedy
\cite{Temlaykov2008} and super greedy  \cite{Liu2012} to quantify
the greedy criterion for  approximation purpose.
However, to the best of our knowledge, only the SGD criterion is employed in greedy learning, since all the results in greedy approximation
\cite{Liu2012,Temlaykov2000,Temlaykov2008} imply that SGD is
superior to other criteria.


$\bullet$ \textit{Iterative format} : this issue focuses on how to
define a new estimator based on the selected atoms. Similar to the
``dictionary selection'', the ``iterative format'' issue is also a
classical topic in greedy approximation. There are several types
of  iterative
schemes \cite{Temlaykov2008}. Among these, three
most commonly used iterative
schemes are   pure greedy \cite{Konyagin1999}, orthogonal
greedy \cite{Devore1996} and relaxed greedy formats \cite{Temlaykov2008a}.
Each iterative format possesses its own
pros and cons \cite{Temlaykov2003,Temlaykov2008}.
For instance, compared with the orthogonal greedy format,  pure and
relaxed greedy formats have benefits in computation but suffer
from either low convergence rate or small applicable scope.

$\bullet$ \textit{Termination rule} : this issue depicts how to
terminate the learning process. The termination  rule is regarded
as the main difference
between greedy approximation and learning, which  has been recently studied 
\cite{Barron2008,Chen2013,Lin2013a,Xu2014}. For example, Barron et
al. \cite{Barron2008} proposed an $l^0$-based complexity
regularization strategy as the termination  rule, and Chen et al. \cite{Chen2013} provided an
$l^1$-based adaptive termination  rule.

Let $H$ be a Hilbert space endowed with norm $\|\cdot\|_H$ and inner
product $\langle\cdot,\cdot\rangle_H$. Let $\mathcal
D=\{g\}_{g\in\mathcal D}$ be a given dictionary satisfying
$\sup_{g\in D,x\in X}|g(x)|\leq 1$. Denote $\mathcal
L_1=\{f:f=\sum_{g\in D}a_gg\}$ as a Banach space endowed with the
norm
\begin{equation}
\|f\|_{\mathcal
    L_1}:=\inf_{\{a_g\}_{g\in\mathcal D}}\left\{\sum_{g\in \mathcal D}|a_g|:f=\sum_{g\in \mathcal
    D}a_gg\right\}.
\end{equation}

There exist several types of  greedy algorithms
\cite{Temlaykov2003}. The three most commonly used are the pure greedy
algorithm (PGA) \cite{Konyagin1999}, orthogonal greedy algorithm (OGA) \cite{Devore1996} and  relaxed
greedy algorithm (RGA) \cite{Temlaykov2008a}.
These algorithms initialize with $f_0:=0$. The new
approximation $f_k\; (k \ge 1)$ is defined based on
$r_{k-1}:=f-f_{k-1}$. In OGA, $f_k$ is defined by
\begin{equation}
f_k=P_{V_{{\bf z},k}} f,
\end{equation}
where $P_{V_{{\bf z},k}}$ is the orthogonal projection onto the space
$V_{{\bf z},k}=\mbox{span}\{g_1,\dots,g_k\}$ and $g_k$ is defined as
\begin{equation}
g_k=\arg\max_{g\in\mathcal D}|\langle r_{k-1},g\rangle_H|.
\end{equation}

Given ${\bf z}=(x_i,y_i)_{i=1}^m$, the empirical inner product and
norm are   defined by
\begin{equation}
\langle f,g\rangle_m:=\frac1m\sum_{i=1}^mf(x_i)g(x_i), 
\end{equation}
and 
\begin{equation}
\|f\|_m^2:=\frac1m\sum_{i=1}^m|f(x_i)|^2.
\end{equation}
Setting $f_{\bf z}^0=0$,  the four aforementioned issues are attended in OGL as follows:

\begin{itemize}
 \item Dictionary selection: Select a suitable dictionary $\mathcal
D_n:=\{g_1,\dots,g_n\}$.

\item Greedy criterion:
\begin{equation}\label{OGAgd}
g_k=\arg\max_{g\in\mathcal D_n}|\langle r_{k-1},g\rangle_m|.
\end{equation}

\item Iteration format:
\begin{equation}
f_{\bf z }^k=P_{V_{{\bf z},k}} f,
\end{equation}
where $P_{V_{{\bf z},k}}$ is the orthogonal projection onto
$V_{{\bf z},k}=\mbox{span}\{g_1,\dots,g_k\}$ in the metric of
$\langle\cdot,\cdot\rangle_m$.

\item   Termination rule: Terminate the learning process
when $k$ satisfies a certain assumption.
\end{itemize}

\section{Greedy criterion in OGL}

Given a real functional $V:  H\rightarrow\mathbf R$, the Fr\'{e}chet
derivative of $V$ at $f$, $V'_f:  H\rightarrow\mathbf R$ is a
linear functional such that for $h\in  H$,
\begin{equation}
\lim_{\|h\|_{
        H}\rightarrow0}\frac{|V(f+h)-V(f)-V'_f(h)|}{\|h\|_{  H}}=0,
\end{equation}
and the gradient of $V$ as a map $\mbox{grad}V: H\rightarrow H$ is
defined by
\begin{equation}
\langle \mbox{grad}V(f),h\rangle_H=V'_f(h),\ \mbox{for
    all}\ h\in  H.
\end{equation}
The greedy criterion adopted in  Eq.(\ref{OGAgd}) is to find $g_k\in \mathcal D_n$
such that
\begin{equation}
\langle
-\mbox{grad}(A_m)(f_{\bf z}^{k-1}),g_k\rangle=\sup_{g\in \mathcal D_n}\langle
-\mbox{grad}(A_m)(f_{\bf z}^{k-1}),g\rangle,
\end{equation}
where $A_m(f)=\sum_{i=1}^m|f(x_i)-y_i|^2$. Therefore, the classical
greedy criterion is based on the steepest  gradient descent (SGD) of
$r_{k-1}$ with respect to the dictionary $\mathcal D_n$. By
normalizing the residual $r_k$, $k=0,1,2,\dots,n$, greedy criterion
in Eq.(\ref{OGAgd}) means to search $g_k$ satisfying
\begin{equation}
g_k=\arg\max_{g\in\mathcal D_n}\frac{|\langle r_{k-1},g\rangle_m|}{\|r_{k-1}\|_m}.
\end{equation}
Geometrically,  the current  $g_k$ minimizes the angle between
$r_{k-1}/\|r_{k-1}\|_m$ and $g$, which is depicted  in  Fig. \ref{IdeaGC}.

Recalling the definition of OGL, it is not difficult to verify
that the angles satisfy
\begin{equation}
|\cos\theta_1|\leq|\cos\theta_2|\leq\cdots\leq|\cos\theta_k|\leq\cdots\leq|\cos\theta_n|,
\end{equation}
or
\begin{equation}
\frac{|\langle r_{0},g_1\rangle_m|}{\|r_{0}\|_m}
\geq
\cdots
\geq
\frac{|\langle r_{k-1},g_k\rangle_m|}{\|r_{k-1}\|_m}
\geq
\cdots
\geq
\frac{|\langle r_{n-1},g_n\rangle_m|}{\|r_{n-1}\|_m},
\end{equation}
since $\frac{|\langle
    r_{k-1},g_k\rangle_m|}{\|r_{k-1}\|_m}=|\cos\theta_k|$. If the
algorithm  stops at the $k$-th iteration, then there exists a
threshold $\delta \in [|\cos\theta_k|,|\cos\theta_{k+1}|]$ to
quantify whether another atom should be added   to construct the
final estimator. To be detailed, if $|\cos{\theta_k}|\geq\delta$,
then $g_k$ is regarded as an ``active atom'' and can be selected to
build the estimator, otherwise, $g_k$ is a ``dead atom '' which
should be discarded.
Based on the above observations and motivated by the Chebshev greedy
algorithm with thresholds  \cite{Temlaykov2008a}, we are
interested in selecting an arbitrary  ``active atom'', $g_k$, in
$\mathcal D_n$, that is
\begin{equation}\label{our metric}
\frac{|\langle r_{k-1},g_k\rangle_m|}{\|r_{k-1}\|_m} >
\delta.
\end{equation}
If there is no $g_k$ satisfying  Eq. (\ref{our metric}),
then the algorithm terminates. We call the greedy criterion  Eq. (\ref{our
metric}) as the ``$\delta$-greedy threshold'' criterion. In practice,
the number of ``active atom'' is usually not unique. We can choose
the first ``active atom'' satisfied  Eq. (\ref{our metric}) at each
greedy iteration to accelerate the algorithm. Once the ``active
atom'' is selected, then the algorithm goes to  the next greedy
iteration and the ``active atom'' is redefined. 

Through such a greedy-criterion, we can develop a new orthogonal greedy learning
scheme, called thresholding orthogonal greedy learning (TOGL). The
two corresponding elements of TOGL can be reformulated as follows:

\begin{itemize}
\item  Greedy definition: Let $g_k$ be an arbitrary (or the first) atom from
$\mathcal D_n$ satisfying  Eq. (\ref{our metric}).


\item Termination rule: Terminate  the learning process
either there is no atom satisfying  Eq. (\ref{our metric}) or $k$
satisfies a certain assumption.
\end{itemize}

Without considering the termination  rule, the classical greedy criterion  Eq. (\ref{OGAgd}) in OGL always selects the greediest  atom at
each greedy iteration. However,
 Eq. (\ref{our metric}) slows down the  speed of gradient descent and
therefore may conduct a more flexible model selection strategy.
According to the bias and variance balance principle
\cite{Cucker2007}, the bias decreases while the variance increases
as a new atom is selected to build the estimator.  If a
lower-correlation atom  is added, then the bias decreases slower and
the variance also increases slower. Then, the  balance  can be
achieved in TOGL   within a  more gradual  flavor than OGL. Moreover, 
 Eq. (\ref{our metric})  also provides a terminate condition that if
all atoms, $g$, in $\mathcal D_n$ satisfy
\begin{equation}\label{Stop 1}
\frac{|\langle r_{k-1},g\rangle_m|}{\|r_{k-1}\|_m}  \le
\delta,
\end{equation}
then the algorithm terminates. The termination  rule concerning
$k$ in TOGL is  necessary and is used to avoid certain extreme cases in practice.
Indeed, using only the  terminate   condition  Eq. (\ref{Stop 1}) may
drive the algorithm to select all atoms from $\mathcal D_n$.  As  Fig. \ref{FSTC1} shows, if the target function $f$ is almost orthogonal to the
space spanned by the dictionary and  atoms in the dictionary  are
almost linear dependent, then the selected
$\delta$ should be too small to distinguish which  is the ``active
atom ''. Consequently, the corresponding learning scheme selects all
 atoms of dictionary and therefore degrades the generalization
capability of OGL.
\begin{figure}[H]
    \centering
    \includegraphics[height=3cm,width=7cm]{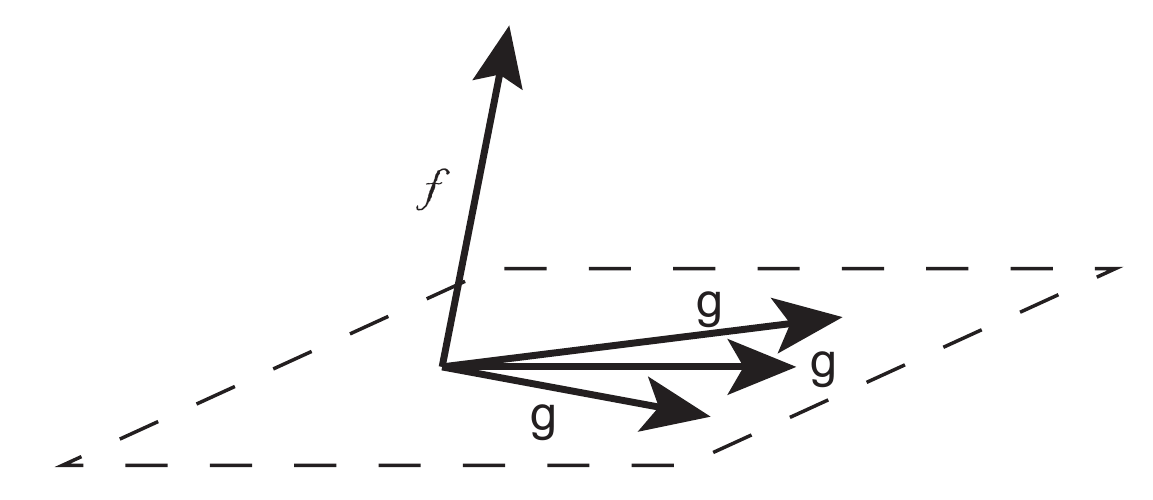}
    \caption{The necessity of termination  rule concerning
    	$k$ in TOGL.}  \label{FSTC1}
\end{figure}
Now we present a theoretical assessment of TOGL.
At first, we give a few notations and concepts, which will be used in the rest part of the paper.
For $r>0$, the space $\mathcal L_{1,\mathcal D_n}^r$ is defined to
be the set of all functions $f$ such that,  there exists a
$h\in\mbox{span}\{\mathcal D_n\}$ satisfying
\begin{equation}\label{prior}
\|h\|_{\mathcal L_1(\mathcal D_n)}\leq\mathcal B, \ \mbox{and}\
\|f - h\| \leq {\mathcal  B}{n^{ - r}},
\end{equation}
where $\|\cdot\|$ denotes the uniform norm for the continuous
function space $C(X)$. The infimum of all  $\mathcal B$  satisfying  Eq. (\ref{prior}) defines
a norm (for $f$ ) on $\mathcal L_{1,\mathcal D_n}^r$.  The   Eq. (\ref{prior}) defines an interpolation
space and is a natural assumption for the regression function in
greedy learning  \cite{Barron2008}. This assumption has already been adopted  to analyze the learning capability
of greedy learning \cite{Barron2008,Lin2013a,Xu2014}. The  Theorem \ref{THEOREM1} illustrates
the performance of TOGL and consequently, reveals the feasibility of
the greedy criterion in  Eq. (\ref{our metric}).

\begin{theorem}\label{THEOREM1}
    Let $0<t<1$, $0<\delta\leq 1/2$,  and $f_{\bf z}^{k,\delta}$  be the
    estimator deduced by TOGL. If $f_\rho\in \mathcal L_{1,\mathcal D_n}^r$, then there
    exits a ${k^*} \in \mathbf N$ such that
    $$
    \begin{aligned}
    & {\cal E}({\pi _M}f_{\bf{z}}^{{k^*},\delta}) - {\cal E}({f_\rho }
    ) \le \\
    & C{{\cal B}^2}({(m{\delta ^2})^{ - 1}}\log m\log\frac{1}{\delta }\log\frac{2}{t}
    + {\delta ^2} + {n^{ - 2r}})
    \end{aligned}
    $$
    holds  with probability at least $1-t$, where $C$ is a positive
    constant depending only on $d$ and $M$.
\end{theorem}

If $\delta= \mathcal O (m^{-1/4})$, and the size of dictionary, $n$,
is selected to be large enough, i.e., $n \geq \mathcal
O({m^{\frac{1}{{4r}}}})$, then
Theorem {\ref{THEOREM1}}
shows that the
generalization error  of ${\pi_M}f_{\bf{z}}^{{k^*},\delta}$ is asymptotic to $\mathcal O
(m^{-1/2}(\log m)^2)$. Up to a logarithmic factor, this bound is the
same as that in \cite{Barron2008} and is the ``record'' of OGL. This
implies that weakening the level of greed in OGL is a feasible way
to avoid traversing the dictionary. It should also be pointed out
that different from OGL \cite{Barron2008}, there are two parameters,
$k$ and $\delta$, in TOGL. Therefore, Theorem {\ref{THEOREM1}} only
presents a theoretical verification that introducing the
``$\delta$-greedy threshold'' to measure the level of greed does not
essentially degrade the generalization capability of OGL. Taking the
practical applications into account, eliminating the condition
concerning $k$ in the termination  rule  is crucial.
This is the scope
of the following section, where an adaptive  termination  rule with
respect to $\delta$ is presented.

\section{$\delta$-thresholding orthogonal greedy learning}

In the previous section, we developed a new greedy learning scheme called as thresholding orthogonal greedy learning (TOGL) and theoretically verified its feasibility. However, there are two main parameters (i.e., the value of threshold $\delta$ and iteration $k$) should be simultaneously fine-tuned. 
It puts more pressure on  parameter selection, which may dampen the spirits of
practitioners.
Given this, we further propose an adaptive
termination  rule only based on the value of threshold.
Notice that,  the value $\|r_{k-1}\|_m/\|y(\cdot)\|_m$ becomes smaller and smaller along the selection of more and more ``active'' atoms, where
$y(\cdot)$ is a function satisfying $y(x_i)=y_i, i=1,\dots,m$. Then,
an advisable  terminate condition is to use  $\delta$ to
quantify $\| r_{k-1}\|_m/\|y(\cdot)\|_m$. Therefore, we append
another terminate condition as
\begin{equation}\label{Our metric2}
\| r_{k-1}\|_m \leq \delta\|y(\cdot)\|_m
\end{equation}
to replace the previous terminate condition concerning $k$ in TOGL.
Based on it,  a new termination  rule can be obtained:
\begin{itemize}
\item  Termination rule: Terminate the learning process if
either Eq. (\ref{Our metric2}) holds or there is no atom satisfying
 Eq. (\ref{our metric}). That is:
    \begin{equation}\label{Our metric3}
    \max_{g\in \mathcal D_n}|\langle r_{k},g\rangle_m|\leq\delta\|r_k\|_{m} \ \text{or} \
    \|r_k\|_m\leq\delta\|f\|_m.
    \end{equation}
\end{itemize}
For such  a change,  we present  a new learning system named the
$\delta$-thresholding orthogonal greedy learning ($\delta$-TOGL)  as
the  Algorithm 1.
%
\begin{algorithm}[H]
    \caption{$\delta$-TOGL}\label{DTOGL}
    \begin{algorithmic}
        \STATE {{\textbf{Step 1 (Initialization)}}:\\
            Given data ${\bf z}=(x_i,y_i)_{i=1}^m$ and dictionary $\mathcal D_n$.\\
            Given a proper greedy threshold $\delta$.\\
            Set initial estimator $f_0=0$ and iteration $k:=0$.}
        \STATE {{\textbf{Step 2 ($\delta$-greedy threshold)}}:\\
             Select $g_k$ be an
            arbitrary atom from $\mathcal D_n$ satisfying
            $$
            \frac{|\langle r_{k-1},g_k\rangle_m|}{\|r_{k-1}\|_m} >
            \delta.
            $$
        }
        \STATE {\textbf{{Step 3 (Orthogonal projection)}}: \\
            Let
            $V_{{\bf z},k} =\mbox{Span}\{g_1,\dots,g_{k}\}$. Compute  $f^{\delta}_{{\bf z}}$ as:
            $$
            {f^{\delta}_{{\bf z}}} = {P_{{\bf z},V_{{\bf z},k}}}({  y}).
            $$
            The residual:
            $r_{k}:=y-f^{\delta}_{{\bf z}},$
            where $P_{{\bf z},V_{{\bf z},k}}$ is the orthogonal projection  onto
            space $V_{{\bf z},k}$ in the criterion of
            $\langle\cdot,\cdot\rangle_m$.}
        \STATE {\textbf{{Step 4 (Termination rule)}}:\\ 
        	If termination rule satisfied as:
            $$
            \max_{g\in \mathcal D_n}|\langle r_{k},g\rangle_m|\leq \delta\|r_k\|_{m} \ \text{or} \
            \|r_k\|_m\leq\delta\|f\|_m,
            $$
            then the algorithm terminates and outputs final estimator $f_{\bf z}^\delta$. \\           
            Otherwise, turn to Step 2 and  $k:=k+1$.
            }
    \end{algorithmic}
\end{algorithm}

	The implementation of OGL requires traversing the dictionary, which has a complexity of $\mathcal O(mn)$.
    Inverting a $k \times k$
	matrix  in orthogonal projection has a complexity of   $\mathcal O(k^3)$. Thus, the $k$th iteration of OGL has a complexity of $ \mathcal O(mn + k^3)$. In Step
	2 of  $\delta$-TOGL, $g_k$ is an arbitrary atom from $\mathcal D_n$
	satisfying the ``$\delta$-greedy threshold'' condition. 
	It motivates us to select the first atom from $\mathcal D_n$
	satisfying  Eq. (\ref{our metric}).
	Then the complexity of
	$\delta$-TOGL is  smaller than $\mathcal O(mn+k^3)$. In fact, it
	usually requires a complexity of $\mathcal O(m+k^3)$, and  gets a
	complexity of $\mathcal O(mn+k^3)$ only for the worst case.
	$\delta$-TOGLR essentially reduces the  complexity of OGL,
	especially when $n$ is large. The memory requirements of OGL and
	$\delta$-TOGL are $\mathcal O(mn)$.

The following theorem shows
that if $\delta$ is appropriately tuned, then the $\delta$-TOGL
estimator $f_{\bf z}^\delta$ can realize the (almost) optimal
generalization capability of OGL and TOGL.

\begin{theorem}\label{THEOREM2}
    Let $0<t<1$, $0<\delta\leq 1/2$,  and $f_{\bf z}^\delta$  be defined
    in Algorithm 1. If $f_\rho\in \mathcal L_{1,\mathcal D_n}^r$, then the inequality
$$
    \begin{aligned}
        &\mathcal E(\pi_Mf_{\bf z}^\delta)-\mathcal E(f_\rho)
        \leq \\
        & C{{\cal B}^2}({(m{\delta ^2})^{ - 1}}\log m
        \log\frac{1}{\delta }\log\frac{2}{t} + {\delta ^2} + {n^{ - 2r}})
    \end{aligned}
$$
    holds  with probability at least $1-t$, where $C$ is a positive
    constant depending only on $d$ and $M$.
\end{theorem}

If $n \geq \mathcal O({m^{\frac{1}{{4r}}}})$  and $\delta=
\mathcal O (m^{-1/4})$, then the learning rate in Theorem  \ref{THEOREM2}
asymptotically equals to $\mathcal O (m^{-1/2}(\log m)^2)$, which is
the same as that of Theorem  \ref{THEOREM1}. Therefore, Theorem
\ref{THEOREM2} implies that using  Eq. (\ref{Our metric2}) to replace the terminate condition concerning $k$  is theoretically feasible.
The most important  highlight of Theorem \ref{THEOREM2} is that it provides a totally different way to circumvent the overfitting phenomenon of OGL. The termination  rule is crucial for OGL, but designing an
effective  termination  rule is a tricky problem.
All the aforementioned studies \cite{Barron2008,Chen2013,Xu2014} of
the termination  rule attempted to design a termination  rule by
controlling the number of iterations directly. Since the
generalization capability of OGL is sensitive to the number of
iterations, the results are at times inadequate. The termination  rule employed in the present paper is based
on the study of the ``greedy-criterion'' issue of greedy learning.
Theorem \ref{THEOREM2} shows that, besides controlling the number of
iterations directly, setting a greedy threshold to redefine the
greedy criterion  can also conduct  an effective termination  rule. 
Theorem  \ref{THEOREM2} implies that this new termination  rule theoretically
works as well as others. Furthermore, when compared with $k$ in OGL,
the generalization capability of the $\delta$-TOGL is stable to
$\delta$,  since the new criterion slows down the changes of bias
and variance.

\section{Simulation verifications}

In this section, a series of simulations  are  carried out to
verify  our theoretical assertions.
Firstly, we introduce the simulation settings, including the data
sets, dictionary, greedy criteria and experimental environment. 
Secondly, we analyze the relationship between the greedy criteria and  generalization performance in orthogonal greedy learning (OGL) and demonstrate that steepest gradient descent (SGD) is not the unique greedy criterion. 
Thirdly, we present a performance comparison of different greedy criteria and
illustrate the ``$\delta$-greedy threshold'' is feasible.
Fourthly,  we empirically study the  performance  of $\delta$-thresholding orthogonal greedy learning ($\delta$-TOGL) and justify the feasibility of it.
Finally,  we compare $\delta$-TOGL with other widely used dictionary-based learning methods and show it is a promising learning scheme.

\subsection{Simulation settings} Throughout the simulations, let ${\bf
    z}=\{(x_{i},y_{i})\}_{i=1}^{m_1}$ be the
training samples with $\{x_i\}_{i=1}^{m_1}$ being drawn
independently and identically according to the uniform distribution
on $[-\pi,\pi]$ and   $y_{i}=f_{\rho }(x_{i})+\mathcal
N(0,\sigma^2), $ where
$$
{f_\rho}(x) = \frac{{\sin x}}{x}, \quad x \in [ - \pi ,\pi ].
$$
Four levels of noise:  $\sigma_1=0.1$, $\sigma_2=0.5$, $\sigma_3=1$
and $\sigma_4=2$ are used in the simulations. The learning
performance (in terms of root mean squared error (RMSE)) of
different algorithms
are then tested by applying the resultant estimators to the test set ${\bf z}_{test}=%
\{(x_{i}^{(t)},y_{i}^{(t)})\}_{i=1}^{m_2}$, which is similarly generated
 as ${\bf z}$ but with a promise that $y_{i}$ are  taken to be $%
y_{i}^{(t)}=f_{\rho }(x_{i}^{(t)}).$

In each simulation,  we use the Gaussian radial basis function (RBF)
\cite{Chen1991} to build up the dictionary:
$$
\left\{e^{-\|x-t_i\|^2/\eta^2}: i=1, \ldots,n\right\},
$$
where $\{t_i\}_{i=1}^n$ are drawn according to the uniform
distribution in  $[-\pi,\pi]$. Since the aim of each simulation is
to compare $\delta$-TOGL with other learning methods on the same
dictionary, we just set $\eta=1$  throughout the simulations.

We use four  different  criteria to select the new atom in each greedy iteration:
$$
{g_k}: = \arg \mathop {\max }\limits_{g \in \mathcal D_n} | \langle {r_{k - 1}},g \rangle_m
|,
$$
$$
{g_k}: = \arg \mathop {\text{second} \max }\limits_{g \in \mathcal D_n} | \langle {r_{k - 1}},g \rangle_m
|,
$$
$$
{g_k}: = \arg \mathop {\text{third} \max }\limits_{g \in \mathcal D_n} | \langle {r_{k - 1}},g\rangle_m |,
$$
and
$$
{g_k} \ \text{randomly  selected  from} \
\mathcal{D}_n.
$$
Here, $\arg \mathop { \text{second} \max }\limits_{}$ and
$\arg \mathop { \text{third} \max }\limits_{}$  mean the values of 
$|\langle r_{k-1},g\rangle_m|$  reach the second and third largest values, respectively. 
Randomly selected means to randomly select $g_k$ from the dictionary. 
We use four abbreviations OGL1,  OGL2, OGL3 and OGLR to 
to denote the corresponding learning schemes, respectively.


Let $\mathcal D_{n,k,\delta}$ be the
set of atoms of $\mathcal D_n$ satisfying $ \frac{|\langle
r_{k-1},g_k\rangle_m|}{\|r_{k-1}\|_m} > \delta$. Four corresponding
criteria are employed as following:
$$
{g_k}: = \arg \mathop {\max }\limits_{g \in \mathcal D_{n,k,\delta}}
| \langle {r_{k - 1}},g \rangle_m |,
$$
$$
{g_k}: = \arg \mathop {\text{second} \max }\limits_{g \in \mathcal
D_{n,k,\delta}} | \langle {r_{k - 1}},g \rangle_m |,
$$
$$
{g_k}: = \arg \mathop {\text{third}  \max }\limits_{g \in \mathcal
D_{n,k,\delta}} | \langle {r_{k - 1}},g\rangle_m |,
$$
and
$$
{g_k}= \text{First}  (D_{n,k,\delta}).
$$

Here $\text{First}  (D_{n,k,\delta})$ denotes the first atom of $\mathcal D_n$
satisfying $ \frac{|\langle r_{k-1},g_k\rangle_m|}{\|r_{k-1}\|_m} > \delta$.
We also use TOGL1 (or $\delta$-TOGL1), TOGL2 (or $\delta$-TOGL2), TOGL3 (or $\delta$-TOGL3) and TOGLR (or $\delta$-TOGLR)  
to denote the corresponding algorithms.


All numerical studies are implemented by MATLAB R2015a on a Windows
personal computer with Core(TM) i7-3770 3.40GHz CPUs and RAM
16.00GB.  All the statistics are averaged based on 10 independent
trails.

\subsection{Greedy criteria in OGL}

In this section, we examine the role of the greedy
criterion in OGL via comparing the performance of OGL1, OGL2, OGL3
and OGLR.  Let $m_1=1000$, $m_2=1000$ and $n=300$ throughout this
subsection. Fig. \ref{OGA}  shows the performance of OGL with four
different greedy criteria.
\begin{figure*}[htb]
    \centering
    \includegraphics[height=5cm,width=19cm]{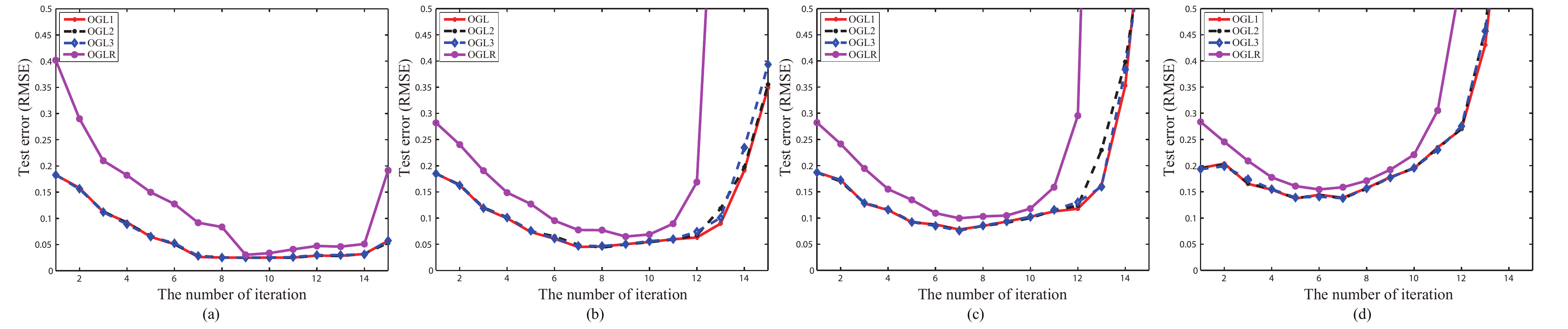}
    \caption{The generalization performance of OGL with
    	four different greedy criteria. (a) The noise level $\sigma_1=0.1$. (b) $\sigma_2=0.5$. (c) $\sigma_3=1$. (d) $\sigma_4=2$.}\label{OGA}
\end{figure*}
We  observe that OGL1, OGL2 and OGL3 have similar performance,
while OGLR performs worse. This shows that SGD is not the
unique greedy criterion and shows the necessity to study the
``greedy criterion'' issue.
 Detailed comparisons   are
listed in the  Table \ref{OGATab}. Here TestRMSE and
$k_{OGL}^*$ denote the theoretically optimal RMSE and number of iteration, where the parameter $k$ is selected according to the test data directly.

\begin{table}[htb]
   \renewcommand{\arraystretch}{1.3}
    \begin{center}
        \caption{ Quantitive comparisons of OGL with   different greedy criteria.}\label{OGATab}
            \begin{tabular}{|c|c|c|c|c|c|} \hline
                Methods & TestRMSE   & ${k_{OGL}^*}$ & Methods & TestRMSE  & ${k_{OGL}^*}$   \\   \hline
                \multicolumn{3}{|c|}{$\sigma=0.1$}  & \multicolumn{3}{|c|}{$\sigma=0.5$}  \\ \hline
                OGL1   &0.0249  &9   &  OGL1   &0.0448  &7 \\ \hline
                OGL2   &0.0248  &9   &  OGL2   &0.0436  &8 \\ \hline
                OGL3   &0.0251  &10  &  OGL3   &0.0466  &8 \\ \hline
                OGLR   &0.0304  &9   &  OGLR   &0.0647  &9 \\ \hline
                Methods & TestRMSE   & ${k_{OGL}^*}$ & Methods & TestRMSE  & ${k_{OGL}^*}$   \\   \hline
                \multicolumn{3}{|c|}{$\sigma=1$}  & \multicolumn{3}{|c|}{$\sigma=2$}  \\ \hline
                OGL1   &0.0780  &7  &   OGL1   &0.1371  &5  \\ \hline
                OGL2   &0.0762  &7  &   OGL2   &0.1374 &7   \\ \hline
                OGL3   &0.0757 &7   &   OGL3   &0.1377 &7   \\ \hline
                OGLR   &0.0995  &7  &   OGLR   &0.1545  &6  \\ \hline
            \end{tabular}
    \end{center}
\end{table}

\subsection{Feasibility of ``$\delta$-greedy threshold''}

In this simulation, we aim at verifying the feasibility of the
``$\delta$-greedy threshold'' criterion.
 For this purpose, we select optimal $k$ according to the test data
directly and compare different greedy criteria satisfying Eq. (\ref{our
metric}). Fig. \ref{TOGA} shows the simulation
results.
\begin{figure*}[htb]
    \centering
    \includegraphics[height=5.2cm,width=19cm]{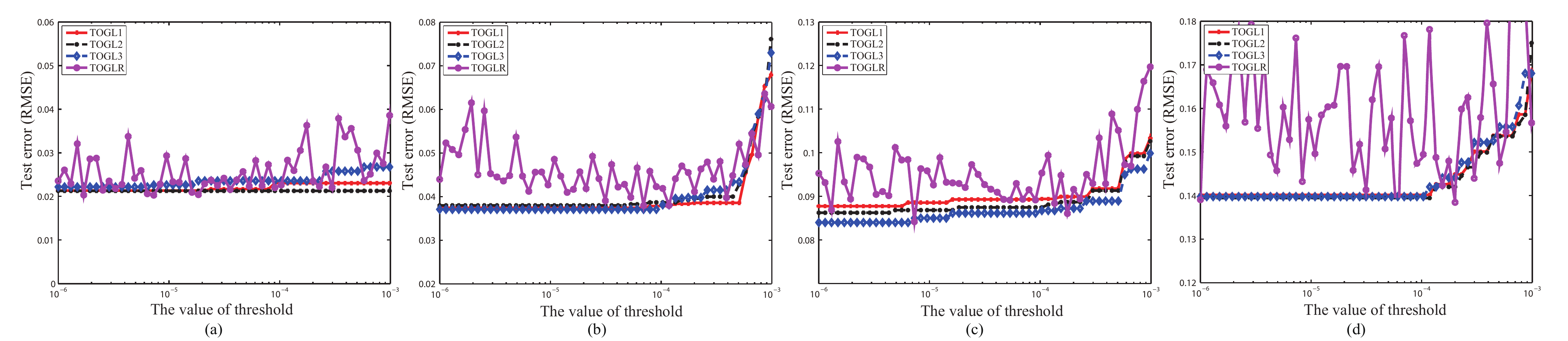}
    \caption{The generalization performance of TOGL with
    	four different greedy criteria. (a) The noise level $\sigma_1=0.1$. (b) $\sigma_2=0.5$. (c) $\sigma_3=1$. (d) $\sigma_4=2$.}\label{TOGA}
\end{figure*}

\begin{figure*}[htb]
	\centering
	\includegraphics[height=5cm,width=19cm]{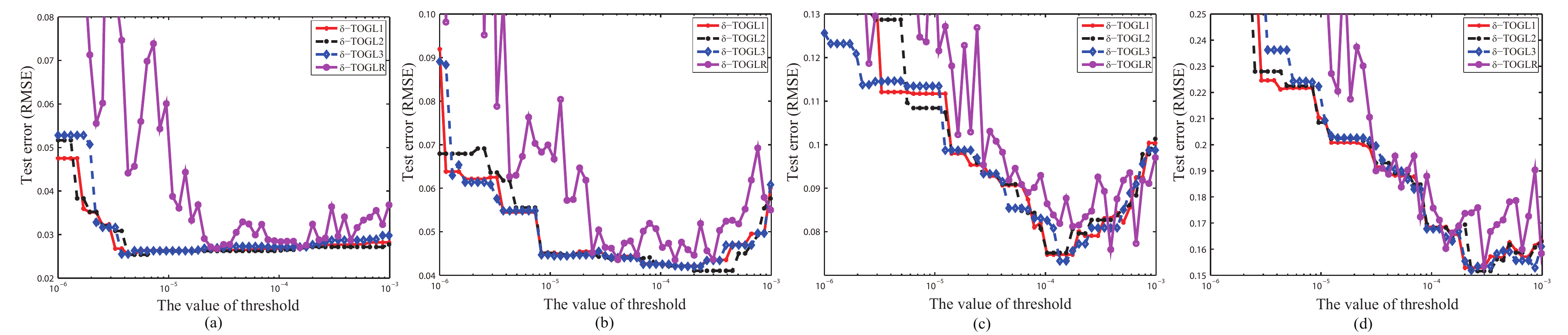}
	\caption{The generalization performance of $\delta$-TOGL with
		four different greedy criteria. (a) The noise level $\sigma_1=0.1$. (b) $\sigma_2=0.5$. (c) $\sigma_3=1$. (d) $\sigma_4=2$.}\label{DTOGA}
\end{figure*}

\begin{figure*}[htb]
	\centering
	\includegraphics[height=5cm,width=19cm]{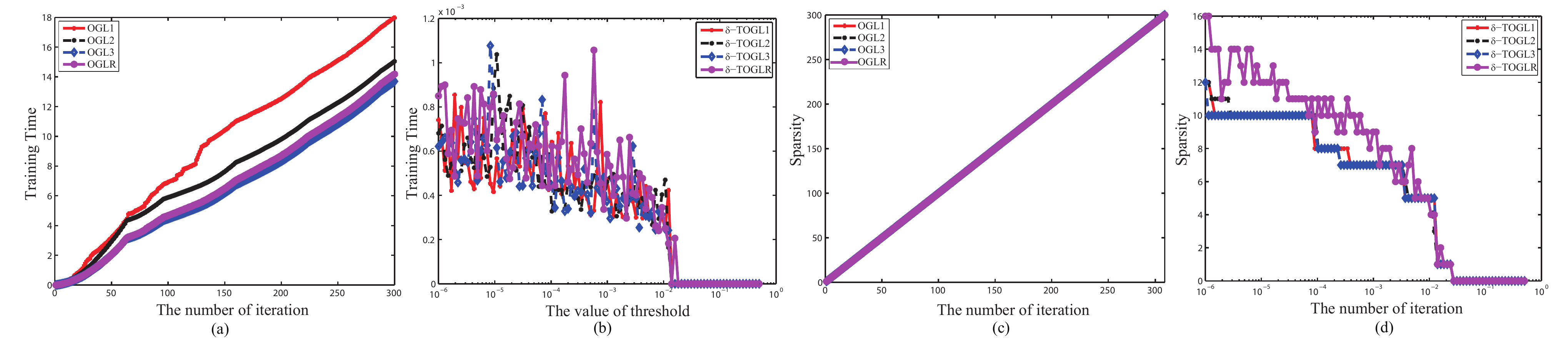}
	\caption{The influence with respect to corresponding parameter changes on the training cost and sparsity  in OGL and  $\delta$-TOGL, respectively. (a) The training time of OGL. (b) The training time of $\delta$-TOGL. (c) The sparsity of the estimator in OGL.  (d) The sparsity of the estimator in $\delta$-TOGL.
	}\label{DTOGL} 
\end{figure*}

Different from the previous simulation,  we   find in this
experiment that the optimal RMSE of TOGLR is similar as that of
TOGL1, TOGL2 and TOGL3. The main reason is that TOGL appends atom
satisfying  the ``$\delta$-greedy threshold'' criterion Eq. (\ref{our
metric}). It implies that once an appropriately value of   $\delta$
is preset, then the selection of the atom is not relevant. Therefore, it
agrees with Theorem  \ref{THEOREM1} and demonstrates that
the introduced ``$\delta$-greedy threshold'' is feasible. We also
present  quantitive comparisons in the  Table \ref{TOGATa}.
\begin{table}[htb]
      \renewcommand{\arraystretch}{1.3}
    \begin{center}
        \caption{Quantitive comparisons  for  different greedy criteria in TOGL.}\label{TOGATa}
        \scalebox{1.1}[1.1]{
        \begin{tabular}{|c|c|c|c|}\hline
            Methods & ${\delta}$ and $k$& TestRMSE &${k_{TOGL}^*}$            \\ \hline
            \multicolumn{4}{|c|}{$\sigma=0.1$} \\ \hline
            TOGL1   &[\text{1.00e-6,3.58e-5}]([9,13])      &0.0213  &8   \\ \hline
            TOGL2   &[\text{1.00e-6,1.70e-6}]([11,12])     &0.0213  &8    \\ \hline
            TOGL3   &[\text{1.00e-6,1.70e-6}]([12,13])     &0.0222  &10   \\ \hline
            TOGLR   &\text{9.52e-6}(12)                    &0.0203  &11   \\ \hline

            \multicolumn{4}{|c|}{$\sigma=0.5$} \\ \hline
            TOGL1   &[\text{1.00e-6,6.95e-5}]([8,13])    &0.0384  &8   \\ \hline
            TOGL2   &[\text{1.00e-6,4.67e-5}]([9,13])    &0.0390  &8   \\ \hline
            TOGL3   &[\text{1.00e-6,9.06e-5}]([8,13])    &0.0371  &8   \\ \hline
            TOGLR   &\text{6.95e-5}(9)                   &0.0379  &8   \\ \hline

            \multicolumn{4}{|c|}{$\sigma=1$} \\ \hline
            TOGL1   &[\text{1.00e-6,5.60e-6}]([11,13]) &0.0877  &8  \\ \hline
            TOGL2   &[\text{1.00e-6,4.30e-6}]([11,13]) &0.0862  &8  \\ \hline
            TOGL3   &[\text{1.00e-6,6.40e-6}]([11,13]) &0.0840  &8  \\ \hline
            TOGLR   &\text{7.30e-6}(12)                &0.0842  &8  \\ \hline

            \multicolumn{4}{|c|}{$\sigma=2$} \\ \hline
            TOGL1   &[\text{1.00e-6,1.18e-4}]([8,13])  &0.1402  &6   \\ \hline
            TOGL2   &[\text{1.00e-6,1.18e-4}]([8,13])   &0.1404  &6  \\ \hline
            TOGL3   &[\text{1.00e-6,1.03e-4}]([8,13]) &0.1408  &6      \\ \hline
            TOGLR   &\text{6.09e-5}(10)                  &0.1392  &5  \\ \hline
        \end{tabular}}
    \end{center}
\end{table}


In Table  \ref{TOGATa}, the  second column (``${\delta}$ and $k$'')
compares the optimal $\delta$  and   corresponding $k$
  (in the bracket) derived  only from
Eq. (\ref{Stop 1}) in TOGL. We also use ${k_{TOGL}^*}$  to denote the
optimal $k$ (with the best performance). The aim of recording these
quantities is to verify that only using Eq. (\ref{Stop 1}) to build up
the terminate criterion is not sufficient. In fact, TABLE
\ref{TOGATa} shows that for some data distributions, Eq. (\ref{Stop 1}) fails to
find out the optimal number of iteration $k$. Compared Table  \ref{TOGATa} with
Table  \ref{OGATab}, we find the TestRMSE derived from TOGL is
comparable with OGL, which states the feasibility of TOGL.

\subsection{Feasibility of  $\delta$-TOGL}

The only difference between $\delta$-TOGL and TOGL lies in the
termination  rule. Firstly we conduct the simulations to verify the
feasibility of the termination  rule Eq. (\ref{Our metric3}) in the  Table  \ref{DTOGAT}.
Here, the second column (${\delta}$ and $k$) records the optimal $\delta$  and 
corresponding $k$   derived from the terminate  rule Eq. (\ref{Our metric3}) in $\delta$-TOGL.
${k_{\delta-TOGL}^*}$  denotes the optimal $k$ selected according to the test samples. We see that the value of $k$
obtained by Eq. (\ref{Our metric3}) is almost the same as
${k_{\delta-TOGL}^*}$ for all four types of noise data. Furthermore, comparing  Table  \ref{DTOGAT} with TABLE
\ref{TOGATa}, we  find their TestRMSE are comparable. All these verify the feasibility and necessity of the termination  rule Eq. (\ref{Our metric3}) in $\delta$-TOGL.

\begin{table}[htb]
      \renewcommand{\arraystretch}{1.3}
    \begin{center}
        \caption{Feasibility of the termination  rule.}\label{DTOGAT}
        \begin{tabular}{|c|c|c|c|}\hline
            Methods & ${\delta}$ and $k$ & TestRMSE & ${k_{\delta-TOGL}^*}$       \\ \hline
            \multicolumn{4}{|c|}{$\sigma=0.1$} \\ \hline
            $\delta$-TOGL1    &[\text{4.30e-6,4.91e-6}](11)    &0.0255 &11   \\ \hline
            $\delta$-TOGL2    &[\text{5.60e-6,6.40e-6}]([10,11])    &0.0254  &10  \\ \hline
            $\delta$-TOGL3    &\text{3.76e-6}(11)              &0.0255  &11   \\ \hline
            $\delta$-TOGLR    &\text{2.75e-5}(11)              &0.0268 &11  \\ \hline

            \multicolumn{4}{|c|}{$\sigma=0.5$} \\ \hline
            $\delta$-TOGL1    &[\text{1.18e-4,1.35e-4}]([7,8])      &0.0407  &7    \\ \hline
            $\delta$-TOGL2    &[\text{2.01e-4,4.45e-4}](7)      &0.0401 &7   \\ \hline
            $\delta$-TOGL3    &[\text{1.54e-4.2.29e-4}]([7,8])      &0.0407  &7   \\ \hline
            $\delta$-TOGLR    &\text{1.35e-4}([8,9])                &0.0406  &9   \\ \hline

            \multicolumn{4}{|c|}{$\sigma=1$} \\ \hline
            $\delta$-TOGL1    &[\text{1.03e-4,1.76e-4}]([7,8])    &0.0747  &7     \\ \hline
            $\delta$-TOGL2    &[\text{1.03e-4,1.54e-4}]([7,8])    &0.0752  &7   \\ \hline
            $\delta$-TOGL3    &[\text{1.35e-4,1.54e-4}]([7,8])    &0.0733  &7   \\ \hline
            $\delta$-TOGLR    &\text{3.89e-4}([7,8])              &0.0759  &7    \\ \hline

            \multicolumn{4}{|c|}{$\sigma=2$} \\ \hline
            $\delta$-TOGL1    &[\text{2.01e-4,2.99e-4}]([6,7])    &0.1529  &6   \\ \hline
            $\delta$-TOGL2    &[\text{2.29e-4,3.41e-4}]([6,7])    &0.1516  &6   \\ \hline
            $\delta$-TOGL3    &\text{2.29e-4}([6,7])              &0.1519  &5   \\ \hline
            $\delta$-TOGLR    &\text{2.99e-4}([7,8])              &0.1537  &6   \\
            \hline
        \end{tabular}
    \end{center}
\end{table}
From OGL to  $\delta$-TOGL, the main parameter changes from $k$ to $\delta$. The following simulations aim at highlighting the role of the  main parameters to illustrate the feasibility of $\delta$-TOGL.
Similar to  Fig. \ref{OGA}, we consider the relation between TestRMSE
and the main parameter of $\delta$-TOGL in the  Fig.
\ref{DTOGA}.
It can be found from Fig. \ref{DTOGA} that although there may be
additional oscillation within a small scope, the generalization
capability of $\delta$-TOGL is not very sensitive to $\delta$ on the whole, which is different from OGL (see Fig. \ref{OGA}).

We also examine the relation between training and test
cost and  the main parameter in OGL and $\delta$-TOGL to illustrate
the feasibility of $\delta$-TOGL. As the test time mainly depends  on
the sparsity of the estimator, we record the sparsity instead. In
this simulation,
the scope
of iterations in OGL starts from $0$ to the size of dictionary
(i.e., $n$=300) and our   theoretical assertions reveal that the
range of  $\delta$ in $\delta$-TOGL is $(0, 0.5]$. We create 50
candidate values of $\delta$ within $[10^{-6},1/2]$. It can be
observed from the results in Fig. \ref{DTOGL} that the training time
(in seconds)  and sparsity of $\delta$-TOGL is far less than OGL,
which implies the computational amount of $\delta$-TOGL is much
smaller than OGL.


%

\subsection{Comparisons}
In this part,  we compare $\delta$-TOGL with   other classical
dictionary-based learning schemes such as the pure greedy learning
PGL \cite{Friedman2001}, OGL \cite{Barron2008},  ridge regression \cite{Golub1979} and Lasso \cite{Tibshirani1995}.
We employ the $\mathcal{L}_2$ regularized least-square (RLS) solution in ridge regression and
the fast iterative shrinkage-thresholding algorithm (FISTA) in Lasso \cite{Beck2009}.
All the parameters, i.e., the number of iterations $k$ in PGL or OGL, the
regularization parameter $\lambda$ in RLS or FISTA and the greedy
threshold $\delta$  in $\delta$-TOGL are all selected according to
test dataset (or test RMSE) directly, since we mainly focus on the
impact of the theoretically optimal parameter rather than validation
techniques. The results are listed in Table  \ref{table6}, where the
standard errors of test RMSE are also reported (numbers in
parentheses).
\begin{table}[htb]
    \renewcommand{\arraystretch}{1.3}
    \begin{center}
        \caption{Comparing the performance of $\delta$-TOGL  with other classic algorithms.}\label{table6}
        \scalebox{0.90}[1]{
        \begin{tabular}{|c|c|c|c|c|}\hline
            Methods & Parameter & TestRMSE  & Sparsity & Running time  \\ \hline
            \multicolumn{5}{|c|}{Regression function $sinc$, dictionary ${\mathcal D}_n, n=300$, noise level  $\sigma=0.1$} \\ \hline
            PGL  &$k=78$&0.0284(0.0037) & 78.0 & 27.4   \\ \hline
            OGL  &$k=9$&0.0218(0.0034) & 9.0   & 11.3\\ \hline
            $\delta$-TOGL1   &\text{$\delta=1.00e-4$}&0.0200(0.0044) &7.4 &4.0  \\ \hline
            $\delta$-TOGL2   &\text{$\delta=2.00e-4$}&0.0203(0.0064) &8.0  &3.9 \\ \hline
            $\delta$-TOGL3   &\text{$\delta=1.30e-6$}&0.0284(0.0074) &12.2 &4.3 \\ \hline
            $\delta$-TOGLR   &\text{$\delta=5.11e-4$}&0.0219(0.0059) &9.1 &3.5  \\ \hline
            $\mathcal L_2$(RLS)     &$\lambda=\text{5e-5}$&0.0313(0.0088) &300.0 &0.5   \\ \hline
            $\mathcal L_1$(FISTA)   &$\lambda = \text{5e-6}$&0.0318(0.0102) &281.2  & 41.7 \\ \hline
            \multicolumn{5}{|c|}{Regression function $sinc$, dictionary ${\mathcal D}_n, n=1000$, noise level  $\sigma=0.1$} \\ \hline
            PGL  &$k=181$&0.0278(0.0044) &181.0  & 116.6  \\ \hline
            OGL  &$k=9$&0.0255(0.0045) &9.0  & 62  \\ \hline
            $\delta$-TOGL1    &\text{$\delta=1.00e-4$}&0.0277(0.0072) &7.2 &5.8 \\ \hline
            $\delta$-TOGL2   &\text{$\delta=6.00e-4$}&0.0294(0.0119) &7.0  &5.8  \\ \hline
            $\delta$-TOGL3   &\text{$\delta=6.00e-6$}&0.0211(0.0036) &7.8  &6.0  \\ \hline
            $\delta$-TOGLR   &\text{$\delta=3.68e-4$}&0.0284(0.0082) &10.4  &4.7  \\ \hline
            $\mathcal L_2$(RLS)  &$\lambda=0.0037$& 0.0322(0.0103)  &1000.0   &6.1 \\ \hline
            $\mathcal L_1$(FISTA)    &$\lambda=\text{8e-6}$&0.0317(0.0079)
            &821.2 & 103.7  \\ \hline
            \multicolumn{5}{|c|}{Regression function $sinc$, dictionary ${\mathcal D}_n, n=2000$, noise level  $\sigma=0.1$} \\ \hline
            PGL  &$k=263$ & 0.0267(0.0036) & 263.0  & 236.4   \\ \hline
            OGL   &$k=9$&0.0250(0.0054) &9.0  &374.7 \\ \hline
            $\delta$-TOGL1    &\text{$\delta=2.00e-4$}&0.0256(0.0078)  &7.1 &9.5  \\ \hline
            $\delta$-TOGL2   &\text{$\delta=1.00e-4$}&0.0280(0.0089) &8.6   &9.3  \\ \hline
            $\delta$-TOGL3   &\text{$\delta=2.00e-6$}&0.0222(0.0082) &7.6   &9.2  \\ \hline
            $\delta$-TOGLR   &\text{$\delta=4.176e-5$}&0.0266(0.0079) &10.6  &6.7  \\ \hline
            $\mathcal L_2$(RLS)   &$\lambda= 0.0005$&0.0305(0.0088) &2000.0  &28.9  \\ \hline
            $\mathcal L_1$(FISTA)    &$\lambda=\text{7e-6}$&0.0335(0.0079) &1252.4 & 176.3
            \\ \hline
        \end{tabular}}
    \end{center}
\end{table}
From the results of Table  \ref{table6},  we observe that
the sparsities (or the number of selected atoms) of greedy-type strategies are far smaller than regularization-based methods, while they enjoy better performance.
It  empirically  verifies that  greedy-type algorithms are more suitable for redundant
dictionary learning, which  is also consistent with \cite{Barron2008}.

Furthermore, it  can be found in Table  \ref{table6} that, although the generalization performance of all the aforementioned learning schemes are similar, $\delta$-TOGL finishes the corresponding learning task within
a remarkably short period of running time. Although PGL has a lower computation complexity than OGL, its convergence rate is quite slow. Generally, PGL needs tens of thousands of iterations  to guarantee performance, just as we preset the maximun of the default number of iteration of PGL is $10000$ in the numerical studies.
Therefore the applicable range of PGL is  restricted.
OGL possesses almost optimal convergence rate and generally converges within a few number of iterations. However, its computation complexity is huge, especially in large-scale dictionary learning.  Table  \ref{table6} shows that, when the size of dictionary $n$ are $300$ and $1000$, OGL performs  faster than PGL, however it is much slower than PGL when $n$ is $2000$.

$\delta$-TOGL can significantly reduce the computation cost of OGL without
sacrificing its generalization performance and sparsity,
just as the results of $\delta$-TOGL1, $\delta$-TOGL2, $\delta$-TOGL3 and $\delta$-TOGLR shown in Table  \ref{table6}. It is mainly due to an appropriate ``$\delta$-greedy threshold''  effective filtering a  mass of ``dead atoms'' from the dictionary.
We also notice that,
$\delta$-TOGLR not only owns the good performance but also has the lowest computation complexity among the four $\delta$-TOGL learning schemes.
It implies that, 
selecting the ``active atom'' from the dictionary without traversal
can further reduce the  complexity without deteriorating 
the performance of OGL.

\section{Real data experiments}

We have verified that $\delta$-TOGL is a feasible learning scheme in 
previous simulations. Especially, $\delta$-TOGLR possesses both  good generalization performance and the lowest computation complexity.
We now  verify the learning performance of $\delta$-TOGLR on five real data sets and compare it with other classical dictionary-based learning methods including PGL, OGL, RLS and FISTA.

The first dataset is  the Prostate cancer dataset
\cite{Blake1998}. The data set consists of the medical records of 97
patients who have received a radical prostatectomy. The
predictors are 8 clinical measures and 1 response variable.
The second  dataset is the Diabetes
data set \cite{Efron2004}. This data set contains 442 diabetes
patients that are measured on 10 independent variables and 1 response variable.
The third  one is the Boston Housing data set created form a housing values survey in suburbs of Boston by Harrison \cite{Harrison1978}. The Boston Housing dataset contains 506
instances which include 13 attributions and 1 response
variable. The fourth one
is the Concrete Compressive Strength (CCS) dataset
\cite{Ye1998}, which contains 1030 instances including 8
quantitative independent variables
and 1 dependent variable. The fifth one is the Abalone dataset\cite{Nash1994} collected for predicting the age of abalone
from physical measurements. The data set contains 4177 instances
which were measured on 8 independent variables  and 1 response variable.

Similarly, we randomly divide all the real data sets into
two disjoint equal parts. The first half serves as the training set
and the second half serves as the test set. We also use the Z-score standardization method \cite{Kreyszig1979} to normalize the data sets, in order to avoid the error caused by considerable  magnitude difference among data dimensions.
For each real data experiment, Gaussian radial basis function is also used to build up the dictionary:
$$
\left\{e^{-\|x-t_i\|^2/\eta^2}: i=1, \ldots,n\right\},
$$
where $\{t_i\}_{i=1}^n$ are drawn as the training samples themselves, thus the size of dictionary equals to training samples.
We set the  standard deviation of radial basis function as $\eta= \frac{d_{max}}{\sqrt{2n}}$, where $d_{max}$ is maximum distance
among all centers $\{t_i\}_{i=1}^n$, in order to
avoid the radial basis function is too sharp or flat.

Table  \ref{t7} documents the experimental results of generalization performance and running time on aforementioned five real data sets. We can clearly observe that, for the small-scale dictionary, i.e., for the Prostate data set, although $\delta$-TOGLR can achieve good performance, its running cost is greater than OGL and RLS. In fact, for each candidate threshold parameter $\delta$, a different iteration of the algorithm is needed run from scratch, which  cancels the computational advantage of $\delta$-TOGLR in small size dictionary learning.
However, we also notice that, for the middle-scale dictionary, i.e., Diabetes,  Housing and CCS, $\delta$-TOGLR begin to gradually surpass the other learning methods in computation with maintaining similar generalization performance as OGL. Especially for the large-scale dictionary learning, i.e., Abalone, $\delta$-TOGLR dominates other methods with a large margin in computation complexity and still possesses good performance.

\begin{table*}[htb]
    \renewcommand{\arraystretch}{1.6}
    \begin{center}
        \caption{The comparative results of performance and running time
            on  five real data sets}\label{t7}
        \scalebox{1}[1]{
            \begin{tabular}{|c|c|c|c|c|c|}\hline
                \backslashbox[2cm] {Methods}{Datasets}  & Prostate  & Diabetes & Housing & CCS & Abalone \\ \hline
              Dictionary size    & $n=50$ & $n=220$    & $n=255$   & $n=520$  & $n=2100$  \\ \hline
                \multicolumn{6}{|c|}{Average  performance } \\ \hline
                $\delta$-TOGLR  & 0.4208 (0.0112) &  55.1226 (1.0347)   & 4.045 (0.4256)   & 7.1279 (0.3294)   & 2.2460 (0.0915)  \\ \hline
                PGL    & 0.4280 (0.0081)  &   56.3125 (2.0542)  & 4.0716 (0.2309)   & 11.2803 (0.0341)  & 2.5880 (0.0106)  \\ \hline
                OGL    & 0.5170 (0.0119)  &  54.6518 (2.8700)   & 3.9447 (0.1139)   & 6.0128 (0.1203)   & 2.1725 (0.0088)  \\ \hline
                RLS    & 0.4415 (0.0951) &   57.3886 (1.5854)   & 3.9554 (0.3236)   & 9.8512 (0.2693)   & 2.2559 (0.0514)  \\ \hline
                FISTA  & 0.6435 (0.0151)  &  61.7636 (2.5811)     & 5.1845 (0.1859)   & 12.8127 (0.3019)  & 3.4161 (0.0774)  \\ \hline

                \multicolumn{6}{|c|}{Average  running time} \\ \hline
                $\delta$-TOGLR  &  0.58   & 1.11   & 0.89   & 0.82  & 4.22  \\ \hline
                PGL    &  41.93  & 49.06  & 52.04  & 79.93 & 193.97  \\ \hline
                OGL    &  0.16    & 1.11  & 1.42   & 7.46  & 787.2  \\ \hline
                RLS    & 0.15 &  0.27     & 0.33   & 1.20  & 42.59  \\ \hline
                FISTA  & 0.52    & 1.11   & 1.40   & 9.04  & 257.8  \\ \hline
            \end{tabular}}
        \end{center}
    \end{table*}

\section{Conclusion and further discussions}

In this paper, we study the greedy criteria in  orthogonal greedy learning (OGL).
The main  contributions can be concluded in  four aspects.

Firstly, we propose that the steepest gradient descent
(SGD) is not the unique greedy criterion to select atoms from dictionary
in OGL, which
paves a new way for exploring greedy criterion in greedy learning.
To the best of our knowledge, this may be the first work concerning  the
``greedy criterion'' issue in the field of supervised learning.
Secondly, motivated by a series of previous researches of Temlyakov and his co-authors in greedy approximation \cite{Temlaykov2000,Temlaykov2003,Temlaykov2008,Temlaykov2008a,Liu2012},
we eventually use the ``$\delta$-greedy threshold'' criterion to quantify the level of greed for the learning purpose.  Our theoretical result shows that OGL with such a
greedy criterion yields a learning rate as $ m^{-1/2} (\log m)^2$,
which is almost the same as that of the classical SGD-based OGL in
\cite{Barron2008}.
Thirdly, based on the ``$\delta$-greedy threshold'' criterion, we
derive an adaptive terminal rule for the corresponding OGL and thus
provide a complete new learning scheme called as $\delta$-thresholding
orthogonal greedy learning ($\delta$-TOGL).
We also present the theoretical demonstration that $\delta$-TOGL can reach the existing (almost) optimal learning rate just as the iteration-based termination  rule dose in  \cite{Barron2008}.
Finally,  we analyze the generalization performance of $\delta$-TOGL
and compare it with other popular dictionary-based learning methods including  pure greedy learning PGL,  OGL,  ridge regression and Lasso  through  plenty of numerical experiments.
The empirical results verify that the $\delta$-TOGL is a promising learning scheme,
which possesses the good generalization performance  and
learns much faster than conventional methods in large-scale dictionary.




%

\appendices
\section{Proofs}
Since Theorem \ref{THEOREM1} can be derived from  Theorem  \ref{THEOREM2} directly, we only prove Theorem \ref{THEOREM2} in
this section. The methodology of proof is somewhat standard in
learning theory. In fact, we   use the error decomposition strategy
in \cite{Lin2013a} to divide the generalization error into
approximation error, sample error and hypothesis error. The main
difficult of the proof is to bound the hypothesis error. The main
tool  to bound it is borrowed from \cite{Temlaykov2008a}.

In order to give an error decomposition strategy for $\mathcal
E(f_{\bf z}^k)-\mathcal E(f_\rho)$, we  need to construct a function
$f_k^*\in \mbox{span}(D_n)$ as follows. Since $f_\rho\in \mathcal
L_{1,\mathcal D_n}^r$, there exists a $h_\rho:=\sum_{i=1}^na_ig_i\in
\mbox{Span}(\mathcal D_n)$ such that
\begin{equation}\label{h}
\|h_\rho\|_{\mathcal L_{1,\mathcal D_n}}\leq\mathcal B,\ \mbox{and}\
\|f_\rho-h_\rho\|\leq \mathcal B n^{-r}.
\end{equation}
Define
\begin{equation}\label{f*}
f_0^*=0,\  f_k^*=\left(1-\frac1k\right)f^*_{k-1}+\frac{\sum_{i=1}^n|a_i|\|g_i\|_\rho}{k}g^*_k,
\end{equation}
where
$$
g_k^*:=\arg\max\limits_{g\in \mathcal D_n'}\left\langle
h_\rho-\left(1-\frac1k\right)f_{k-1}^*,g\right\rangle_{\rho},
$$
and
$$
\mathcal D_n':=\left\{{g_i(x)}/{\|g_i\|_\rho}\right\}_{i=1}^n
\bigcup
\left\{-{g_i(x)}/{\|g_i\|_\rho}\right\}_{i=1}^n
$$
with $g_i\in \mathcal D_n$.

Let $f_{\bf z}^\delta$  and $f_k^*$ be defined as in Algorithm 1 and
Eq. (\ref{f*}), respectively,  then we have
\begin{eqnarray*}
	&&\mathcal E(\pi_Mf_{\bf z}^\delta)-\mathcal E(f_\rho)\\
	&\leq&
	\mathcal E(f_k^*)-\mathcal E(f_\rho)
	+
	\mathcal E_{\bf z}(\pi_Mf_{\bf z}^\delta)-\mathcal E_{\bf z}(f_k^*)\\
	&+&
	\mathcal
	E_{\bf z}(f_k^*)-\mathcal E(f_k^*)+\mathcal E(\pi_Mf_{\bf z}^\delta)-\mathcal
	E_{\bf z}(\pi_Mf_{\bf z}^\delta),
\end{eqnarray*}
where $\mathcal E_{\bf
	z}(f)=\frac1m\sum_{i=1}^m(y_i-f(x_i))^2$.

Upon making the short hand notations
$$
\mathcal D(k):=\mathcal E(f_k^*)-\mathcal E(f_\rho),
$$
$$
\mathcal S({\bf z},k,\delta):=\mathcal
E_{\bf z}(f_k^*)-\mathcal E(f_k^*)+\mathcal E(\pi_Mf_{\bf z}^\delta)-\mathcal
E_{\bf z}(\pi_Mf_{\bf z}^\delta),
$$
and
$$
\mathcal P({\bf z},k,\delta):=\mathcal E_{\bf z}(\pi_Mf_{\bf z}^\delta)-\mathcal E_{\bf
	z}(f_k^*)
$$
respectively for the approximation error, the sample error and the
hypothesis error, we have
\begin{equation}\label{error decomposition}
\mathcal E(\pi_Mf_{\bf z}^\delta)-\mathcal E(f_\rho)=\mathcal
D(k)+ \mathcal S({\bf z},k,\delta)+\mathcal P({\bf z},k,\delta).
\end{equation}

At first, we give an upper bound estimate for $\mathcal D(k)$, which
can be found in Proposition 1 of \cite{Lin2013a}.

\begin{lemma}\label{LEMMA1}
	Let $f_k^*$ be defined in Eq. (\ref{f*}). If
	$f_\rho\in \mathcal L_{1,\mathcal D_n}^r$, then
	\begin{equation}\label{approximation error estimation}
	\mathcal D(k)\leq  \mathcal B^2(k^{-1/2}+n^{-r})^2.
	\end{equation}
\end{lemma}

To bound the sample and hypothesis errors, we need the following
Lemma \ref{LEMMA2}.

\begin{lemma}\label{LEMMA2}
	Let $y(x)$ satisfy $y(x_i)=y_i$,  and $f_{\bf z}^\delta$ be  defined
	in Algorithm 1. Then, there are at most
	\begin{equation}\label{Estimate k}
	C\delta^{-2}\log\frac1\delta
	\end{equation}
	atoms selected to build up the estimator $f_{\bf z}^\delta$.
	Furthermore, for any $h \in \mbox{Span}\{D_n\}$, we have
	\begin{equation}\label{estimate hypothesis error}
	\|y - f_{\bf z}^\delta\|_m^2\leq2\|y - h\|_m^2+
	2\delta^2\|h\|_{\mathcal L_1(\mathcal D_n)}.
	\end{equation}
\end{lemma}

\begin{proof}
	(\ref{Estimate k}) can be found in \cite[Theorem
	4.1]{Temlaykov2008a}. Now we turn to prove (\ref{estimate hypothesis
		error}). Our termination  rule guarantees that either $\max_{g\in
		\mathcal D_n}|\langle r_{k},g\rangle_m|\leq\delta\|r_k\|_{m}$ or
	$\|r_k\|\leq\delta\|y\|_m.$ In the latter case the required bound
	follows form
	\begin{equation*}
	\begin{aligned}
	\|y\|_m & \leq\|y-h\|_m+\|h\|_m \\ & \leq\delta(\|y-h\|_m+\|h\|_m)
	\\ & \leq\delta(\|f-h\|_m+\|h\|_{\mathcal L_1(\mathcal D_n)}).
	\end{aligned}
	\end{equation*}

	Thus, we assume  $\max_{g\in \mathcal D_n}|\langle
	r_{k},g\rangle_m|\leq\delta\|r_k\|_{m}$ holds. By using
	$$
	\langle y-f_k,f_k\rangle_m=0,
	$$
	we have
	\begin{equation*}
	\begin{aligned}
	\|r_k\|_m^2 &=
	\langle r_k,r_k\rangle_m \\
	&= \langle r_k,y-h\rangle_m+\langle r_k,h\rangle_m
	\\ & \leq
	\|y-h\|_m\|r_k\|_m+\langle r_k,h\rangle_m\\
	& \leq
	\|y-h\|_m\|r_k\|_m+\|h\|_{\mathcal L_1(\mathcal D_n)}\max_{g\in \mathcal D_n}\langle
	r_k,g\rangle_m
	\\ & \leq
	\|y-h\|_m\|r_k\|_m+\|h\|_{\mathcal L_1(\mathcal D_n)}\delta\|r_k\|_m.
	\end{aligned}
	\end{equation*}
	This finishes the proof.
\end{proof}

Based on Lemma \ref{LEMMA2} and the fact $\|f^*_k\|_{\mathcal
	L_1(\mathcal D_n)}\leq \mathcal B$ \cite[Lemma 1]{Lin2013a}, we
obtain
\begin{equation}\label{hypothesis error estimation}
\mathcal P({\bf z},k,\delta)\leq 2\mathcal E_{\bf z}(\pi_Mf_{\bf z}^\delta)-\mathcal E_{\bf
	z}(f_k^*)\leq 2\mathcal B\delta^2.
\end{equation}

Now, we turn to    bound the sample error $\mathcal S({\bf z},k)$.
Upon using the short hand notations
$$
S_1({\bf z},k):=\{\mathcal E_{\bf
	z}(f_k^*)-\mathcal E_{\bf
	z}(f_\rho)\}-\{\mathcal E(f_k^*)-\mathcal
E(f_\rho)\}
$$
and
$$
S_2({\bf z},\delta):=\{\mathcal E(\pi_Mf_{\bf z}^\delta)-\mathcal E(f_\rho)\}-\{\mathcal E_{\bf
	z}(\pi_Mf_{\bf z}^\delta)-\mathcal E_{\bf z}(f_\rho)\},
$$
we write
\begin{equation}\label{sample decomposition}
\mathcal S({\bf z},k)=\mathcal S_1({\bf z},k)+\mathcal
S_2({\bf z},\delta).
\end{equation}
It can be found in Proposition 2 of \cite{Lin2013a} that
for any $0<t<1$, with confidence
$1-\frac{t}2$,
\begin{equation}\label{S1 estimate}
\mathcal S_1({\bf z},k)\leq \frac{7(3M+\mathcal B\log\frac2t)}{3m}+\frac12\mathcal D(k)
\end{equation}

Using \cite[Eqs(A.10)]{Xu2014} with $k$ replaced by
$C\delta^{-2}\log\frac1\delta$, we have
\begin{equation}\label{S2 estimate}
\mathcal S_2({\bf z},\delta)\leq \frac12\mathcal E(\pi_Mf_{\bf
	z}^\delta)-\mathcal
E(f_\rho)+\log\frac2t\frac{C\delta^{-2}\log\frac1\delta\log
	m}{m}
\end{equation}
holds with confidence at least $1-t/2$. Therefore, (\ref{error
	decomposition}), (\ref{approximation error estimation}),
(\ref{hypothesis error estimation}), (\ref{S1 estimate}), (\ref{S2
	estimate}) and (\ref{sample decomposition}) yields that
\begin{equation*}
\begin{aligned}
& \mathcal E(\pi_Mf_{\bf z}^\delta)-\mathcal E(f_\rho)
\\ & \leq
C\mathcal B^2( (m\delta^2)^{-1}\log m\log \frac{1}{\delta }\log\frac2t+\delta^2+n^{-2r})
\end{aligned}
\end{equation*}
holds with confidence at least $1-t$. This finishes the proof of
Theorem \ref{THEOREM2}.

\section*{Acknowledgment}
The research was supported by National Basic Research Program (973 Program) (2013CB329404) and Key Project of National Science Foundation of China  (Grant No. 11131006 and 91330204).
\end{document}